\documentclass[10pt, twoside]{article}
\usepackage{ifthen}
\usepackage[cp1251]{inputenc}
\usepackage[english,russian]{babel}
\usepackage{makeidx}  % allows for indexgeneration
\usepackage{amsmath,amssymb}
\usepackage{amsthm}
\usepackage[mathscr]{euscript}
\usepackage{cite}
\usepackage{url}
\usepackage{graphicx}
\usepackage{epstopdf}
\usepackage{multirow}
\usepackage[tablename=Table,figurename=Figure]{caption}
\usepackage[left=16mm,right=16mm,papersize={172mm,239mm},top=17mm,bottom=22mm]{geometry}

\newlength{\ind}
\parindent=\ind

\newtheoremstyle{newthmstyle}{3pt}{3pt}{\itshape}{\ind}{\bfseries}{.}{.5em}{}
\newtheoremstyle{newdefstyle}{3pt}{3pt}{}{\ind}{\bfseries}{.}{.5em}{}

\theoremstyle{newthmstyle}
\newtheorem{thm}{Theorem}

\newtheorem{cor}{Corollary}

\theoremstyle{newdefstyle}
\newtheorem{defn}{Definition}

\renewenvironment{proof}{P r o o f.}{\qed}

\renewenvironment{itemize}{\begin{list}{$\bullet$}{
\setlength{\parsep}{0mm}
\setlength{\itemsep}{0mm}
\setlength{\topsep}{0mm}
\setlength{\itemindent}{0mm}
\setlength{\labelsep}{2mm}
\setlength{\labelwidth}{4mm}
\setlength{\leftmargin}{\labelwidth}
\addtolength{\leftmargin}{\parindent}
}}{\end{list}}

\newcounter{en}
\DeclareCaptionLabelSeparator{point_separ}{. }  % раздеалитель - точка
\DeclareCaptionLabelFormat{no_index}{#1}    %без нумерации, т.к. рисунок один
\def\R{\mathbb R}

\begin{document}

\begin{center}
%{\small \noindent 2018 \hfill Вестник Санкт-Петербургского университета. \hfill Т.~14. Вып.~4 }\\
%{\normalsize ПРИКЛАДНАЯ МАТЕМАТИКА. ИНФОРМАТИКА. ПРОЦЕССЫ УПРАВЛЕНИЯ}
\end{center}

\

%{\large \noindent ПРИКЛАДНАЯ МАТЕМАТИКА}

\

{\footnotesize \noindent {UDC 519.854}}

{\footnotesize \noindent {MSC 90C27, 90C29, 90C59}}

\

\noindent {\bf Construction and reduction of the Pareto set \\ in asymmetric travelling salesman problem with two criteria$^*$}

\

\noindent \emph{A. O. Zakharov $^1$, Yu. V. Kovalenko $^{2}$}
{
\renewcommand{\thefootnote}{}\footnote{\:
$^*$ This work was supported by Russian Foundation for Basic Research (project N 17-07-00371)
and by the Ministry of Science and Education of the Russian Federation (under the 5-100 Excellence Programme).}
\addtocounter{footnote}{-1}
}

\noindent {\small $^1$ St. Petersburg State University, 7-9, Universitetskaya nab., \\
199034, St. Petersburg, Russian Federation \\
$^2$ Novosibirsk State University, 1, Pirogova ul.,
630090, Novosibirsk, Russian Federation
}

\

\

{\footnotesize \noindent We consider the bicriteria asymmetric travelling salesman problem (bi-ATSP).
Optimal solution to a multicriteria problem is usually supposed to be the Pareto set,
which is rather wide in real-world problems.
For the first time we apply to the bi-ATSP  the axiomatic approach of the Pareto set reduction proposed by V.~Noghin.
We identify series of ``quanta of information'' that guarantee the reduction of the Pareto set for particular cases of the bi-ATSP.
An approximation of the Pareto set to the bi-ATSP is constructed by a new multi-objective genetic algorithm.
The experimental evaluation carried out in this paper
shows the degree of reduction of the Pareto set approximation for various ``quanta of information''
and various structures of the bi-ATSP instances generated randomly or from TSPLIB problems.}

{\footnotesize \noindent \emph{Keywords:}
reduction of the Pareto set, decision maker preferences, multi-objective genetic algorithm, computational experiment.}

\

\textbf{1. Introduction.} The asymmetric travelling salesman problem (ATSP) is one of the most popular problems in combinatorial optimization~[1].
Given a complete directed graph, where each arc is associated
with a positive weight, we search for a circuit visiting every vertex of the graph exactly once and minimizing the total weight.
In this paper we consider the bicriteria ATSP (bi-ATSP), which is a special case
of the multicriteria ATSP~[2], where an arc is associated to a couple
of weights.

The best possible solution to a multicriteria optimization problem (MOP) is usually supposed
to be the Pareto set~[2, 3],
which is rather wide in real-world problems, and difficulties arise in choosing the final variant.
For that reason numerous methods introduce some mechanism to treat the MOP:
utility function, rule, or binary relation, so that methods are aimed at finding
an ``optimal'' solution with respect to this mechanism.
However, some approaches do not guarantee that the obtained solution will be from the Pareto set.
State-of-the-art methods are the following~[4]:
multiattribute utility theory, outranking approaches, verbal decision analysis,
various iterative procedures with man-machine interface, etc.
We investigate the axiomatic approach of the Pareto set reduction
proposed in monograph~[5] which has an alternative idea.
Here the author introduced an additional information about the decision maker (DM) preferences in terms of
the so-called ``quantum of information''.
The method shows how to construct a new bound of the optimal choice, which is narrower than the Pareto set.
Practical applications of the approach could be found in works~[6, 7].

As far as we know, the axiomatic approach of the Pareto set reduction has not been widely investigated
in the case of discrete optimization problems,
and an experimental evaluation has not been carried out on real-world instances.
Thus, we apply this approach to the bi-ATSP in order to estimate its effectiveness,
i.e. the degree of the Pareto set reduction and how it depends on the parameters of the information about DM's preferences.
We identify series of ``quanta of information'' that guarantee the reduction of the Pareto set
 with particular structures. The bi-ATSP instances with such structures are presented.
\looseness=-1

Originally the reduction is constructed with respect the Pareto set of the considered problem.
Due to the strongly NP-hardness of the bi-ATSP we take an approximation of the Pareto set in computational experiments.
The ATSP cannot be approximated
with any constant or exponential approximation factor already with a single objective function~[1].
Moreover, in~[8], the non-approximability bounds were obtained for the multicriteria ATSP with weights 1 and 2.
The results are based on the non-existence of a small size approximating set.
Therefore, metaheuristics, in particular multi-objective evolutionary algorithms (MOEAs),
are appropriate to approximate the Pareto set of the bi-ATSP.

Numerous MOEAs have been proposed to MOPs
(see~[9--14]).
There are three main classes of approaches to develop MOEAs, which are known as
Pareto-dominance based (see e.g. SPEA2~[14],  NSGA-II~[9, 10], NSGA-III~[12]),
decomposition based (see e.g. MOEA/D~[11]) and indicator based approaches (see e.g. SIBEA~[13]).
NSGA-II~[10] has one of the best results in the literature on
multi-objective genetic algorithms (MOGAs) for the MOPs with two or three
objectives.
In article~[9], a fast implementation of a steady-state version of NSGA-II is proposed for two dimensions.

In~[15, 16] NSGA-II was adopted to the multicriteria symmetric travelling salesman problem, and
the experimental evaluation was performed on symmetric instances from TSPLIB library~[17].
As far as we know, there is no an adaptation of NSGA-II to the problem, where arc weights are non-symmetric.
In this paper we develop a MOGA based on NSGA-II to solve the bi-ATSP
using adjacency-based representation of solutions. The previous study was communicated in~[18].
The MOGA from~[18] applies a problem-specific heuristic to generate the initial population
and uses a mutation operator, which performs a random jump within 3-opt neighborhood.
In comparison to the MOGA from~[18]
proposed two new crossover operators, where the Pareto-dominance is involved.
A computational experiment is carried out on instances  generated randomly or from ATSP-instances of TSPLIB library.
The preliminary results of the experiment indicate that our MOGA demonstrates competitive results.
The main experimental evaluation shows the degree of the reduction of the Pareto set approximation
 for various structures of the ATSP instances in the case of one and two ``quanta of information''.
Note that early we experimentally investigate the reduction of the Pareto set approximation only in the case of one ``quantum of information'' in~[18].
\looseness=-1

\textbf{2. Problem statement.}
In the $m$-criteria travelling salesman problem~[1] (m-TSP), we are given a complete
weighted graph $G = (V, E)$,
with $V=\{v_1,\dots,v_n\}$ being the set of vertices (nodes), and
set $E$ contains arcs (or edges) between every pair of vertices in $V$.
A weight $d(e)$ assigns to each arc (edge)~$e$ a vector~$(d_1(e),\dots,d_m(e))$ of length $m$,
where each element $d_j(e),\ j=1,\dots,m,$ corresponds to a certain measure of arc (edge)~$e$
like travel distance, travel time, expenses,
the number of flight changes between corresponding vertices (nodes).
The aim is to find  ``minimal'' Hamiltonian circuit(s) of the graph, i.e. closed
tour(s) visiting each of the $n$ vertices of $G$ exactly once. Here ``minimal'' refers to
the notion of Pareto relation.
If graph $G$ is undirected, we have Symmetric m-TSP (m-STSP).
If $G$ is a directed graph, then we have Asymmetric m-TSP (m-ATSP).

The total weight of a tour $C$ is a vector $D(C)=(D_1(C),\dots,D_m(C))$,
where $D_j(C) =\sum_{e\in C} d_j(e),\ j=1,\dots,m$.
We say that one solution (tour) $C^\ast$ dominates another solution $C$
if the inequality $D(C^\ast) \leq D(C)$ holds. The notation $D(C^\ast) \leq D(C)$ means
that $D(C^\ast) \neq D(C)$ and $D_i(C^\ast) \leqslant D_i(C)$
for all $i \in I$, where $I = \{1, 2, \ldots, m \}$. This relation $\leq$ is also called {\it the Pareto relation}.
We denote by $\mathcal{C}$ all possible $(n-1)!$ tours of graph $G$.
A set of non-dominated solutions is called {\it  the set of pareto-optimal solutions}~[2, 3]
$P_D(\mathcal{C})=\{ C \in \mathcal{C} \mid \nexists C^\ast \in \mathcal{C}: D(C^\ast)~\leq~D(C)~\}.$
In discrete problems, the set of pareto-optimal solutions is non-empty if the set of feasible solutions is non-empty,
which is true for the m-TSP.
If we denote $\mathcal{D} = D(\mathcal{C})$, then the Pareto set is defined as
$P(\mathcal{D})=\{ y \in \mathcal{D} \mid \nexists y^\ast \in \mathcal{D}: y^\ast~\leq~y~\}.$
We assume that the Pareto set is specified except for a collection of equivalence classes,
generated by equivalence relation $C' \sim C''$ iff $D(C') = D(C'')$.

In this paper we investigate the issue of the Pareto set reduction for the bi-ATSP.

\textbf{3. Pareto set reduction.} Axiomatic approach of the Pareto set reduction is applied to both discrete and continuous problems.
Due to consideration of the multicriteria ATSP we formulate the basic concepts and results of the approach in terms of notations
introduced in section~2.
Further, we investigate properties of the bi-ATSP in the scope of the Pareto set reduction.

\textbf{\textit{3.1. Main approach.}}
According to~[5] we consider the extended multicriteria problem $<\mathcal{C}, D, \prec>$:
\begin{itemize}
\item a set of all possible $(n-1)!$ tours $\mathcal{C}$;
\item a vector criterion $D = (D_1, D_2, \ldots,  D_m)$ defined on set $\mathcal{C}$;
\item an asymmetric binary preference relation of the DM $\prec$ defined on set $\mathcal{D}$.
\end{itemize}
The notation $D(C') \prec D(C'')$ means that the DM prefers the solution $C'$ to $C''$.

Binary relation $\prec$ satisfies some axioms of the so-called ``reasonable'' \ choice, according which
it is irreflexive, transitive,
invariant with respect to a linear positive transformation and compatible with each criteria $D_1, D_2, \ldots, D_m$.
The compatibility means that the DM is interested in
decreasing value of each criterion, when values of other criteria are constant.
Also, if for some
feasible solutions $C', \ C'' \in \mathcal{C}$
the relation $D(C') \prec D(C'')$ holds, then tour $C''$
does not belong to the optimal choice within the whole set $\mathcal{C}$.

Author~[5] established the Edgeworth--Pareto principle:
under axioms of ``reasonable'' choice any set of selected outcomes $Ch(\mathcal{D})$
belongs to the Pareto set $P(\mathcal{D})$.
Here the set of selected outcomes is interpreted as some abstract set
corresponded to the set of tours, that satisfy all hypothetic preferences of the DM.
So, the optimal choice should be done within the Pareto set only
if preference relation $\prec$ fulfills the axioms of ``reasonable'' choice.

In real-life multicriteria problems the Pareto set is rather wide.
For this reason V.~Noghin proposed a specific information on the DM's preference relation $\prec$ to reduce
the Pareto set staying within the set of selected outcomes~[5, 19]:

\begin{defn}
\label{ref_def_i_j}
We say that there exists a {\it ``quantum of information''} about the DM's preference
relation $\prec$ if vector $y' \in \R^m$ such that
$y'_i = - w_i < 0$, $y'_j = w_j > 0$,
$y'_s = 0$ for all $s \in I\setminus \{i, j\}$
satisfies the expression $y' \prec 0_m$.
In such case we will say, that the component of criteria $i$ is more important than the component~$j$
with given positive parameters~$w_i$,~$w_j$.
\end{defn}

Thus, ``quantum of information'' shows, that the DM is ready to compromise
by increasing the criterion $D_j$ by amount $w_j$ for decreasing the criterion $D_i$ by amount $w_i$.
The quantity of relative loss is set by the so-called
{\it coefficient of relative importance} $\theta = w_j / (w_i + w_j)$, therefore $\theta \in (0, \ 1)$.

As mentioned before the relation $\prec$ is invariant with respect to a linear positive transformation.
Hence Definition~\ref{ref_def_i_j} is equivalent to the existence of such vector $y'' \in \R^m$
with components $y''_i = \theta - 1$, $y''_j = \theta$, $y''_s = 0$ for all $s \in I\setminus \{i, j\}$,
that the relation $y'' \prec 0_m$ holds. Further, in experimental study (subsection~5.2)
we consider ``quantum of information'' exactly in terms of coefficient $\theta$.

In~[5] the author established the rule of taking into account ``quantum of information''.
This rule consists in constructing a ``new'' vector criterion using the components of the ``old'' one
and parameters of the information $w_i$, $w_j$. Then one should find the Pareto set of ``new'' multicriteria problem
with the same set of feasible solutions and ``new'' vector criterion.
The obtained set will belong to the Pareto set of the initial problem
and give a narrower upper bound on the optimal choice as a result the Pareto set will be reduced.

The following theorem states the rule of applying ``quantum of information'' and
specifies how to evaluate ``new'' vector criterion upon the ``old'' one.

\begin{thm}[see {[5]}]
\label{th_red_1}
Given a ``quantum of information'' by Definition~\ref{ref_def_i_j},
the inclusions $Ch(\mathcal{D}) \subseteq \hat{P}(\mathcal{D}) \subseteq P(\mathcal{D})$ are valid
for any set of selected outcomes $Ch(\mathcal{D})$.
Here $\hat{P}(\mathcal{D}) = D(P_{\hat{D}}(\mathcal{C}))$ and $P_{\hat{D}}(\mathcal{C})$ is the set of pareto-optimal solutions
with respect to $m$-dimensional vector criterion $\hat{D} = (\hat{D}_1, \ldots, \hat{D}_m)$, where
$\hat{D}_j = w_j D_i + w_i D_j$, $\hat{D}_s = D_s$ for all $s \neq j$.
\end{thm}

\begin{cor}[see {[5]}]
\label{cl_red}
The result of Theorem~\ref{th_red_1} holds as well as component $j$ of ``new'' vector criterion is evaluated by formula
$\hat{D}_j = \theta D_i + (1-\theta) D_j$.
\end{cor}

Thus ``new'' vector criterion $\hat{D}$ differs from the ``old'' one
only by less important component $j$.
In~[20--22] one can find results
on applying particular collections of ``quanta of information'' and
scheme to arbitrary collection.

Suppose we consider two ``quanta of information'' simultaneously: the $i$-th criterion $D_i$ is more important
than the $j$-th criterion $D_j$ with coefficient of relative importance $\theta_{ij}$,
and the $j$-th criterion $D_j$ is more important than the $i$-th criterion $D_i$
with coefficient of relative importance $\theta_{ji}$.
Such conflicting situation occurs only if the inequality $\theta_{ij} + \theta_{ji} < 1$ holds
(see the explanation in~[5]).
The latter guarantees the existence of two vectors $y^{(1)}, y^{(2)} \in \R^m$:
\begin{equation}
\label{quanta_i_j}
\begin{split}
y^{(1)}_i = \theta_{ij} - 1, \ y^{(1)}_j  & = \theta_{ij}, \ y^{(1)}_s = 0; \\
y^{(2)}_i = \theta_{ji},  \ y^{(2)}_j = \theta_{ji} - 1, & \ y^{(2)}_s = 0 \quad \forall s \in I\setminus \{i, j\},
\end{split}
\end{equation}
such that the relations $y^{(1)} \prec 0_m$, $y^{(2)} \prec 0_m$ are valid~[5].

The following theorem shows how to apply two ``quanta of information''.

\begin{thm}[see {[5]}]
\label{th_red_2}
Given two ``quanta of information'':  the $i$-th criterion $D_i$ is more important than
the $j$-th criterion $D_j$ with coefficient of relative importance $\theta_{ij}$,
and the $j$-th criterion $D_j$ is more important than the $i$-th criterion $D_i$
with coefficient of relative importance $\theta_{ji}$.
The inequality $\theta_{ij} + \theta_{ji} < 1$ is valid.
Then the inclusions $Ch(\mathcal{D}) \subseteq \hat{P}(\mathcal{D}) \subseteq P(\mathcal{D})$ hold
for any set of selected outcomes $Ch(\mathcal{D})$.
Here $\hat{P}(\mathcal{D}) = D(P_{\hat{D}}(\mathcal{C}))$, and $P_{\hat{D}}(\mathcal{C})$ is the set of pareto-optimal solutions
with respect to $m$-dimensional vector criterion $\hat{D} = (\hat{D}_1, \ldots, \hat{D}_m)$, where
$\hat{D}_i = (1-\theta_{ji}) D_i + \theta_{ji} D_j$, $\hat{D}_j = \theta_{ij} D_i + (1-\theta_{ij}) D_j$,
$\hat{D}_s = D_s$ for all $s \neq i, j$.
\looseness=-1
\end{thm}

\textbf{\textit{3.2. Pareto set reduction in bi-ATSP.}}
Here we consider the bi-ATSP and its properties with respect to reduction of the Pareto set.

Obviously, the upper bound on the cardinality of the Pareto set $P(\mathcal{D})$ is $(n-1)!$,
and this bound is tight~[23].
Authors~[24] established the maximum number of elements in the Pareto set
for any multicriteria discrete problem,
that in the case of the bi-ATSP gives the following upper bound:
$|P(\mathcal{\mathcal{D}})| \leqslant \min\{l_1, l_2\}$, where $l_i$ is the number of different values in
the set $\mathcal{D}_i = D_i(\mathcal{C})$, $i = 1, 2$.
In the case of the bi-ATSP with integer weights we get
$l_i \leqslant \max\{\mathcal{D}_i\} - \min\{\mathcal{D}_i\} + 1$,
where values $\max\{\mathcal{D}_i\}$ and $\min\{\mathcal{D}_i\}$ can be replaced
by upper and lower bounds on the objective function ${D}_i$, $i = 1, 2$. %, respectively.

Now, we go to establish theoretical results estimating the degree of the Pareto set reduction.
Let us consider the case, when all elements of the Pareto set lay on principal diagonal of some rectangle in the criterion space.

\begin{thm}
\label{prop1_crit_theta}
Let $P(\mathcal{D}) = \{ (y_1, y_2): y_2 = a - k y_1, y_1 \in \mathcal{D}_1, y_2 \in \mathcal{D}_2 \}$,
where $a$ and $k$ are arbitrary positive constants.
Suppose the 1st criterion $D_1$ is more important than the 2nd one $D_2$
with coefficient of relative importance~$\theta'$. If $\theta' \geqslant k/(k+1)$, then
the reduction of the Pareto set $\hat{P}(\mathcal{D})$ consists of only one element.
In the case of $\theta' < k/(k+1)$ the reduction does not hold, i.e. $\hat{P}(\mathcal{D}) = P(\mathcal{D})$.
\end{thm}
\begin{proof}
The proof is based on geometrical representation of the Pareto set reduction using cone dominance.

Following~[5] the domination of some vector $\hat{y} \in \R^2$ over another vector $\tilde{y} \in \R^2$
by the Pareto relation $\leq$, i.e. $\hat{y} \leq \tilde{y}$,
means that their difference $\hat{y} - \tilde{y}$ belongs
to convex cone $\R_{-}^2 = \mbox{cone}\{-e^1, -e^2 \}\setminus\{0_2\}$ (the non-positive orthant),
where $e^1$ and $e^2$ are unit vectors of space $\R^2$.
In other words, vector $\hat{y}$ dominates vector $\tilde{y}$ by cone $\R_{-}^2$.
Thus, the Pareto set is actually the set of non-dominated vectors with respect to non-positive orthant.

A ``quantum of information'' 1st criterion $D_1$ is more important than the 2nd one $D_2$
in terms of coefficient $\theta'$ is defined by vector $y'$
with components $y'_1 = \theta' - 1$, $y'_2 = \theta'$. Such vector $y'$ extends the non-positive orthant
to convex cone $M = \mbox{cone}\{ -e^1, -e^2,  y' \}\setminus\{0_2\}$,
and set $\hat{P}(\mathcal{D})$ from Theorem~\ref{th_red_1} is the set of non-dominated vectors with respect to cone~$M$.

Vector $(-1, \ k)^T$ gives the direction to line $y_2 = a - k y_1$ in the criterion space.
If vector $(-1, \ k)^T$ belongs to a convex cone without $\{0_2\}$,
generated by vectors $-e^1$, $-e^2$, and $y'$,
then all vectors except one of the Pareto set $P(\mathcal{D})$ will be dominated
with respect to that cone. It is easy to check that the inclusion
$(-1, \ k)^T \in \mbox{cone}\{-e^1, -e^2, y'\} \setminus \{0_2\}$
is valid iff $\theta' \geqslant k/(k+1)$.
\end{proof}

\begin{thm}
\label{prop2_crit_theta}
Let in Theorem~\ref{prop1_crit_theta}, otherwise,
the 2nd criterion $D_2$ is more important than the 1st one $D_1$
with coefficient of relative importance $\theta''$.
Then the reduction of the Pareto set $\hat{P}(\mathcal{D})$ has only one element if $\theta'' \geqslant 1/(k+1)$,
and $\hat{P}(\mathcal{D}) = P(\mathcal{D})$ if $\theta'' < 1/(k+1)$.
\end{thm}

The proof of Theorem~\ref{prop2_crit_theta} is analogous to the proof of Theorem~\ref{prop1_crit_theta}.

Particularly, if the feasible set $\mathcal{D}$ lay on the line $y_2 = a - k y_1$,
we have $P(\mathcal{D}) = \mathcal{D}$,
and the conditions of Theorems~\ref{prop1_crit_theta} and~\ref{prop2_crit_theta} hold.
In such case we say, that {\it criteria $D_1$ and $D_2$ contradict each other with coefficient~$k$}.

Obviously, for any bi-ATSP instance there exists the minimum number of parallel lines with a negative slope, that all elements of the Pareto set belong to them. Thus we have

\begin{cor}
\label{cl_crit_theta}
Let $P(\mathcal{D}) = \bigcup_{i=1}^p \{ (y_1, y_2): y_2 = a_i - k y_1, y_1 \in \mathcal{D}_1, y_2 \in \mathcal{D}_2 \}$,
where $a_i$, $i = 1, \ldots, p$, and $k$ are arbitrary positive constants.
And criterion $D_1$ is more important than criterion $D_2$
with coefficient of relative importance~$\theta'$ and $\theta' \geqslant k/(k+1)$,
or criterion $D_2$ is more important than criterion $D_1$
with coefficient of relative importance~$\theta''$ and $\theta'' \geqslant 1/(k+1)$,
then $|\hat{P}(\mathcal{D})| \leqslant p$.
\end{cor}

\begin{thm}
\label{thm3_crit_theta}
Let $P(\mathcal{D}) = \{ (y_1, y_2): y_2 = a - k y_1, y_1 \in \mathcal{D}_1, y_2 \in \mathcal{D}_2 \}$,
where $a$ and $k$ are arbitrary positive constants.
Suppose the 1st criterion $D_1$ is more important than the 2nd one $D_2$
with coefficient of relative importance~$\theta'$, and the 2nd criterion $D_2$ is more important than the 1st one $D_1$
with coefficient of relative importance $\theta''$. If at least one inequality
$\theta' \geqslant k/(k+1)$ or $\theta'' \geqslant 1/(k+1)$ is valid, then
the reduction of the Pareto set $\hat{P}(\mathcal{D})$ consists of only one element.
If both inequalities $\theta' < k/(k+1)$ and $\theta'' < 1/(k+1)$ hold, then the reduction does not occur, i.e. $\hat{P}(\mathcal{D}) = P(\mathcal{D})$.
\end{thm}

The proof of Theorem~\ref{thm3_crit_theta} is analogous to the proof of Theorem~\ref{prop1_crit_theta}.
Here two ``quanta of information'' are defined by two vectors $y^{(1)}$, $y^{(2)}$ with components~(\ref{quanta_i_j}),
where $m = 2$, $i = 1$, and $j = 2$, that extend the non-positive orthant
to convex cone $M = \mbox{cone}\{ -e^1, -e^2,  y^{(1)}, y^{(2)} \}\setminus\{0_2\}$.

Now we consider the bi-ATSP(1,2), where each arc $e \in E$ has a weight $d(e)= (d_1(e), d_2(e)) \in \{1, \ 2\}^2$
such that $d_2(e) = 3 - d_1(e)$.
Angel et al.~[8] proved that for any $r \geqslant 1$ and any $k \geqslant 5$ there exist
such instances of the bi-ATSP(1,2) with $n = k r$ vertices
that the Pareto set contains at least $r+1$ elements of the form $(2kr, kr)$, $(2kr-n, kr+k)$,
$(2kr-2k, kr+2k)$, $\ldots$, $(kr, 2kr)$.
Thus, criteria $D_1$ and $D_2$ contradict each other with coefficient $1$. % in these instances.
According to Theorems~\ref{prop1_crit_theta}~and~\ref{prop2_crit_theta}
in this class of the bi-ATSP we could exclude at least $r$ elements from the Pareto set
when using ``quantum of information'' the 1st criterion $D_1$ is more important than the 2nd criterion $D_2$,
or vice versa, with coefficient $\theta = 0.5$.

Emelichev and Perepeliza~[23] constructed bi-ATSP instances with integer weights,
where criteria contradict each other with coefficient 1 and $|P(\mathcal{D})| = (n-1)!$.
So, for these instances $|\hat{P}(\mathcal{D})| = 1$ if the 1st criterion $D_1$ is more important than the 2nd criterion $D_2$,
or vice versa, with coefficient $\theta = 0.5$.

Further, we identify the condition that guarantees excluding at least one element from the Pareto set.
Suppose that a ``quantum of information'' is given: the 1st component of criteria $D_1$ is more important than the 2nd one $D_2$
with coefficient of relative importance $\theta$.
We suppose, that there exist such tours $C', C'' \in P_D(\mathcal{C})$ satisfying the following inequality:
\begin{equation}
\label{prop_2}
\left(D_1(C') - D_1(C'')\right) / \left(D_2(C'') - D_2(C')\right) \geqslant \left(1 - \theta\right) / \theta,
\end{equation}
then at least one element will be excluded from the Pareto set $P(\mathcal{D})$
after using a ``quantum of information''.
So, $|P(\mathcal{D})| - |\hat{P}(\mathcal{D})| \geqslant 1$.
Analogous result could be obtained, when the 2nd criterion is more important than the 1st one.
The difficulty in checking inequality~(\ref{prop_2}) is that we should know two elements of the Pareto set.
Meanwhile the tours $C_{\mbox{min}1} = \mbox{argmin}\{D_1(C), \ C \in \mathcal{C}\}$,
$C_{\mbox{min}2} = \mbox{argmin}\{D_2(C), \ C \in \mathcal{C}\}$ are pareto-optimal by definition.

The results of this subsection are true for any discrete bicriteria problem.

\textbf{4. Multi-objective genetic algorithm.}
The genetic algorithm is a random search method
that models a process of evolution of a population of {\em
individuals}~[25]. Each individual is a sample solution
to the optimization problem being solved.
Individuals of a new population are built by means of reproduction
operators (crossover and/or mutation).

\textbf{\textit{4.1. NSGA-II scheme.}}
To construct an approximation of the Pareto set
to the bi-ATSP we develop a MOGA based on Non-dominated Sorting Genetic Algorithm II (NSGA-II)~[10].
The NSGA-II is initiated by generating $N$ random solutions of the initial population.
Next the population is sorted based on the non-domination relation (the Pareto relation).
Each solution $C$ of the population is
assigned a {\em rank} equal to its {\em non-domination level}
(1 is the best level, 2 is the next best level, and so on).
The first level consists of all non-dominated solutions of the population.
Individuals of the next level are the non-dominated solutions of the population, where
solutions of the first level are discounted, and so on.
Ranks of solutions are calculated in $O(mN^2)$ time by means of the algorithm proposed in~[10].
We try to minimize rank of individuals in evolutionary process.
To get an estimate of the density of solutions surrounding a solution $C$ in a non-dominated level of the population,
two nearest solutions on each side of this solution are identified for each of the objectives.
The estimation of solution $C$ is called {\em crowding distance} and
it is computed as a normalized perimeter of the cuboid formed in the criterion space by
the nearest neighbors.
The crowding distances of individuals in all non-dominated levels
are computed in $O(mN\mathrm{log}N)$ time (see e.g.~[10]).

The NSGA-II is characterized by the population management strategy known as generational model~[25].
At each iteration of NSGA-II we select $N$ pairs of parent solutions from the current population~$P_{t-1}$.
Then we mutate parents, and create offspring, applying a crossover (recombination) to each pair of parents.
Offspring compose population $Q_{t-1}$.

The next population $P_{t}$ is constructed from the best $N$ solutions of the current population~$P_{t-1}$ and
the offspring population $Q_{t-1}$. %created from $P_{t-1}$ by applying selection, crossover, and mutation
Population $Q_{t-1}\cup P_{t-1}$ is sorted based on the non-domination relation, and the crowding distances
of individuals are calculated.
The best $N$ solutions are selected using the rank and the crowding distance.
Between two solutions with differing non-domination ranks, we prefer the solution with the lower rank.
If both solutions belong to the same level, then we prefer the solution with the bigger crowding distance.
\looseness=-1

One iteration of the presented NSGA-II is performed in $O(mN^2)$ time as shown in~[10].
In our implementation of the NSGA-II
four individuals of the initial population are constructed by a problem-specific heuristic presented in~[26]
for the ATSP with one criterion. The heuristic first solves
the Assignment Problem, and then patches the circuits of the optimum
assignment together to form a feasible tour in two ways.
So, we create two solutions with each of the objectives.
All other individuals of the initial population are generated randomly.

Each parent  is chosen by $s$-{\em tournament %on Step~2.1.1
selection}: sample randomly $s$~individuals from the current
population and select the best one by means of the rank and the crowding distance.
\looseness=-1

\textbf{\textit{4.2. Recombination and mutation operators.}}
The experimental results of~[26, 27] for the m-TSP indicate
that  reproduction operators with the adjacency-based representation of solutions have an
advantage over operators, which emphasize the order or position of the vertices in parent solutions.
We suppose that a feasible solution to the bi-ATSP  is encoded as
a list of arcs.

In the recombination operator we use one of new crossovers $\mbox{DEC}_{\mbox{PR}}$ or $\mbox{DPX}_{\mbox{PR}}$ proposed here.
The operator $\mbox{DEC}_{\mbox{PR}}$ (Directed Edge Crossover with Pareto Relation)
may be considered as a  ``direct descendant'' of the well-known EX operator (Edge Crossover),
and the operator $\mbox{DPX}_{\mbox{PR}}$ (Distance Preserving Crossover with Pareto Relation) is
a  ``direct descendant'' of the well-known DPX operator (Distance Preserving Crossover).
EX and DPX were originally developed for the 1-STSP with single-objective~[27].

Both operators $\mbox{DEC}_{\mbox{PR}}$ and $\mbox{DPX}_{\mbox{PR}}$ are {\em respectful}~[28],
i.e. all arcs shared by both parents are copied into the offspring.
Moreover, we try to construct an offspring of good quality,
taking into account the Pareto relation.
To this end, the tour fragments are reconnected using a modification of the nearest neighborhood heuristic, where at each step
we choose a non-dominated feasible arc. The Pareto relation on the set of arcs is defined similar
to the Pareto relation on  the set of solutions.
Feasible arcs in $\mbox{DEC}_{\mbox{PR}}$ are the ones, that are contained in at least one of the parents.
However, in $\mbox{DPX}_{\mbox{PR}}$ the feasible set consists of the arcs, that are absent in parents.
Note that new arcs are inserted taking into account the non-violation of sub-tour elimination constraints.

$\mbox{DEC}_{\mbox{PR}}$ operator is exploitive, but it can never generate new arcs (transmitting requirement~[28]).
So, we use mutation operators to introduce new arcs and therefore diversity into the MOGA  populations.
By contrast, $\mbox{DPX}_{\mbox{PR}}$ operator is explorative because it not only inherits common arcs from parents,
but also introduces new arcs.

If the offspring obtained by $\mbox{DEC}_{\mbox{PR}}$ or $\mbox{DPX}_{\mbox{PR}}$ is equal to one of the parents,
then the result of the recombination is calculated by applying the well-known {\em shift mutation}~[28]
to one of the two parents with equal probability.
This approach allows us to avoid creating a clone of parents and to maintain a diverse set of solutions in the population.
\looseness=-1

The mutation is also applied to each parent  with
probability $p_{\mbox{mut}}$, which is a tunable parameter of the MOGA.
We use a mutation operator proposed in~[26] for the one-criteria ATSP.
It performs a random jump within 3-opt neighborhood,
trying to improve a parent solution in terms of one of the criteria.
Each time one of two objectives is used in mutation with equal probability.
\looseness=-1

\textbf{5. Computational experiment.}
This section presents the results of the computational experiment on the bi-ATSP instances.
Our MOGA was programmed in C++ and tested on a %(NSGA-II-biATSP)
computer with Intel~Core~i5~3470 3.20~GHz processor, 4 Gb~RAM.
We set  the tournament size~${s=10}$ and the mutation probability~${p_{\mbox{mut}}=0.1}$.
We use the notation NSGA-II-DPX (NSGA-II-DEC) for the MOGA employing the $\mbox{DPX}_{\mbox{PR}}$
($\mbox{DEC}_{\mbox{PR}}$) crossover.

Various meta-heuristics and heuristics have been developed for the m-STSP,
such as Pareto local search algorithms, MOEAs, multi-objective ant colony optimization methods,
memetic algorithms and others (see, e.g.,~[15, 16, 29--31]).
However, we have not found in the literature any multi-objective metaheuristic proposed specifically
to the  m-ATSP and experimentally tested on instances with non-symmetric weights of arcs.
In~[18], we proposed new MOGA based on NSGA-II to solve the bi-ATSP, but no a crossover taking
into account bi-criteria nature of the problem was used,
and a detailed experimental evaluation of the MOGA was not performed in~[18].
Note that the performance of MOGAs depends significantly upon the choice of the crossover operator.

So, the computational experiment consists of two stages.
At the first stage estimated the performance of NSGA-II-DPX and NSGA-II-DEC on instances of small sizes, %NSGA-II-biATSP
for which the Pareto sets are found by an exact algorithm.
At the second stage the degree of reduction of the Pareto set approximation is evaluated.

We note that there exists MOOLIBRARY library~[32],
which contains test instances of some discrete multicriteria problems.
However, the m-TSP is not presented in this library,
so we generate the bi-ATSP test instances randomly and
construct them from the ATSP instances of TSPLIB library~[17], as well.
\looseness=-1

\textbf{\textit{5.1. Pareto set approximation.}}
NSGA-II-DPX and NSGA-II-DEC were tested using small-size problem instances of four series with $n=12$:
S12[1,10][1,10], S12[1,20][1,20], S12[1,10][1,20], S12contr[1,2][1,2].
Each series consists of five problems with integer weights $d_1(\cdot)$ and $d_2(\cdot)$
of arcs randomly generated from intervals specified at the ending of the series name.
In series S12contr[1,2][1,2] the criteria contradict each other with coefficient~$1$, i.e. weights are generated so that $d_2(e)=3-d_1(e)$ for all $e\in E$.
The Pareto sets to the considered bi-ATSP~instances were found by complete enumeration of $(n-1)!$ possible Hamiltonian circuits.
The population size $N$  was set to $50$ on the basis of the preliminary experiments.
Our MOGAs were run $30$ times for each instance and each run continued for $1000$ iterations.
\looseness=-1

Let $P^*:=P(\mathcal{D})$ be the Pareto set, and
$A$ be its approximation obtained by NSGA-II-DPX or NSGA-II-DEC.
In order to evaluate the performance of the proposed algorithms and compare them,
the generational distance (GD)~[12] and
the inverted generational distance (IGD)~[12] are involved as performance metrics:
$$
\mbox{GD}(A,P^*)=\frac{1}{|A|}\sqrt{\sum_{i=1}^{|A|}\mu_i^2},\ \
\mbox{IGD}(A,P^*)=\frac{1}{|P^*|}\sqrt{\sum_{i=1}^{|P^*|}\tilde{\mu}_i^2},
$$
where $\mu_i$ ($\tilde{\mu}_i$) is the Euclidean distance between the $i$-th
member (two-dimensio\-nal point) in the set $A$ ($P^*$) and its nearest member in the
set $P^*$ ($A$). GD can only reflect the convergence of an algorithm,
while IGD could measure both convergence and
diversity. Smaller value of the metrics means better quality.

The results of experiment are presented in Table~\ref{tab:Ser1_10}.
Here index In corresponds to the values at the initial population and
index Fin corresponds to the values at the final population in average over $30$ runs.
$N^{P^*}$ represents the number of elements in the Pareto set.
\looseness=-1

\begin{table}[!h] \centering
\caption{\textbf{Approximation of the Pareto set for series}}%~  ($30$ runs, $1000$ iterations)}
\label{tab:Ser1_10}
{\footnotesize
\begin{tabular}{|c|c|c|c|c|c|c|c|c|c|c|}
\hline
&\multicolumn{4}{|c|}{NSGA-II-DPX}&\multicolumn{4}{|c|}{NSGA-II-DEC}&\\
\cline{2-9}
Inst & $\mbox{GD}_{\mbox{In}}$ & $\mbox{IGD}_{\mbox{In}}$ & $\mbox{GD}_{\mbox{Fin}}$ & $\mbox{IGD}_{\mbox{Fin}}$ & $\mbox{GD}_{\mbox{In}}$ & $\mbox{IGD}_{\mbox{In}}$& $\mbox{GD}_{\mbox{Fin}}$ & $\mbox{IGD}_{\mbox{Fin}}$ & $N^{P^*}$\\
\hline
\multicolumn{10}{|c|}{{\it Series} S12[1,10][1,10]}\\
\hline
1 &  8.109 & 4.89   & 1.214 & 1.205  & 7.515 &  4.686  & {\bf 1.003} &  {\bf 1.045} & 13 \\
\hline
2 &  4.188 & 4.481  & 1.693 & 1.3  &   3.944 &  4.543  & {\bf 1.17}  &  {\bf 0.929} & 10 \\
\hline
3 &  2.614 & 4.224  & 1.519 & 1.489  &  3.791 &  4.243  & {\bf 1.24}  &  {\bf  1.224}& 10\\
\hline
4 &  8.511 & 5.289 & 1.724 &  1.375  & 8.114 &  5.115  & {\bf 1.426} &  {\bf 1.001} & 16 \\
\hline
5 &  7.405 & 3.768    & 1.263 &  1.093   & 7.864 & 3.911  &  1.171 &  1.092& 18\\
\hline
Aver &   6.165 & 4.53 & 1.483 & 1.293 & 6.245 &  4.499 &  1.202 & 1.058 & 13.4 \\
\hline
\multicolumn{10}{|c|}{{\it Series} S12[1,20][1,20]}\\
\hline
1 &   17.843 & 7.82 &   1.685 &  1.511 & 16.42  & 7.45  &    {\bf 1.404} & {\bf 1.337} &  22 \\
\hline
2 &  16.049 & 7.857 &  3.063 &  2.246 & 15.619 & 7.64  &    {\bf 1.987} & {\bf 1.524} &  26 \\
\hline
3 &  15.483 & 7.114 &  2.86 &   2.452 &  15.888 & 7.272 &    {\bf 1.973} & {\bf 1.832} & 24 \\
\hline
4 &  16.531 & 7.39 &    2.296 &  1.676 &  16.129 & 7.165 &    {\bf 1.259} & {\bf 1.072} &  27\\
\hline
5 &  14.468 & 7.152 &  2.643 &  2.484 & 15.294 & 7.485 &    {\bf 1.855} & {\bf 1.824} &  23 \\
\hline
Aver &   16.075  &   7.467  &   2.509 & 2.074  &  15.87 &   7.402 & 1.696 & 1.518   & 24.4 \\
\hline
\multicolumn{10}{|c|}{{\it Series} S12[1,10][1,20]}\\
\hline
1 &  10.732 & 8.427  & 2.835 & 2.2   & 10.76 &  8.343 & {\bf 2.488} & {\bf 1.882} & 11 \\
\hline
2 &  1.994 &  4.802  & 1.686 & 1.478  & 1.143 &  4.935 & {\bf 0.63} &  {\bf 0.915} & 14 \\
\hline
3 &  0     &   4.17  & 1.405 & 1.729  & 0     &  4.17  & {\bf 0.441} & {\bf 1.659} & 6  \\
\hline
4 &  10.621 & 6.17   & 1.965 & 1.714  & 10.508 & 6.2   & {\bf 1.075} & {\bf 1.092} & 21 \\
\hline
5 &  9.444 &  6.034  & 2.455 &  1.942 & 9.448  & 6.098 & {\bf 1.814} & {\bf 1.428} & 17 \\
\hline
Aver &   6.558    & 3.787 & 2.069  & 1.813  &  6.372 & 5.949 & 1.289 & 1.395 & 13.8 \\
\hline
\multicolumn{10}{|c|}{{\it Series} S12contr[1,2][1,2]}\\
\hline
1  &  0  &  0.201 &    0 &  0 & 0 & 0.202 &  0 &  0 & 13 \\
\hline
2  &  0  &  0.199 &    0 &  0 & 0 & 0.199 & 0 &  0 & 13 \\
\hline
3  &  0  & 0.202 &     0 &  0 & 0 & 0.208 & 0 &  0 & 13 \\
\hline
4  &  0  & 0.249 &     0 &  0 & 0 & 0.221 & 0 &  0 & 13 \\
\hline
5  &  0  &  0.217 &    0 &  0 & 0 & 0.201 & 0 & 0  & 13\\
\hline
Aver  &  0 &  0.213 & 0 &  0 & 0 &  0.206 & 0 &  0 & 13\\
\hline
\end{tabular}
\\
Note. The average final result that is significantly better than the other is marked in bold
(according the Wilcoxon signed-rank test).
}
\end{table}

As seen from Table~\ref{tab:Ser1_10},
on all series except for S12contr[1,2][1,2],
for NSGA-II-DPX the distance GD decreases in approximately $5$ times  and
the distance IGD -- in approximately $3.5$ times on average from the initial population to the final one.
At the same time, for NSGA-II-DEC the average distance GD decreases in approximately $6.5$ times  and
the average distance IGD -- in approximately $4.5$ times.
Moreover, the average number of elements in the Pareto set approximation
increases approximately $3$ times during $1000$ iterations for both algorithms.
The values of performance metrics  at the final population show the convergence of
the approximation obtained by NSGA-II-DEC or NSGA-II-DPX to the Pareto set.
The average CPU time of one trial for both MOGAs  is approximately 2 minutes on all instances.

In  instances of series~S12contr[1,2][1,2] any feasible solution is pareto-optimal.
So, the main purpose of the MOGA is to obtain all elements of the Pareto set,
and the Pareto dominance is not important.
Both considered MOGAs find the Pareto set on all trials and on all instances in less than $100$ iterations.

The statistical analysis of experimental data was carried out
using the Wilcoxon signed-rank test~[12] at a $5\%$ significance level.
We test for each instance the difference between values of metric $\mbox{GD}_{\mbox{Fin}}$ (or $\mbox{IGD}_{\mbox{Fin}}$)
for algorithms NSGA-II-DEC and NSGA-II-DPX over $30$ trials.
The average final result for a metric that is significantly better than the other is marked in bold in Table~\ref{tab:Ser1_10}.
In 15 out of 20 considered instances,
NSGA-II-DEC outperforms NSGA-II-DPX (in 14 of 15 cases the differences between values of metrics are statistically significant).
So, we can conclude that our MOGAs demonstrate competitive results, but
$\mbox{DEC}_{\mbox{PR}}$ crossover has an advantage over $\mbox{DPX}_{\mbox{PR}}$ crossover.
Therefore, in subsection 5.2 we use only NSGA-II-DEC to find
an approximation of the Pareto set. Further research may include construction of a MOGA,
where crossovers $\mbox{DEC}_{\mbox{PR}}$ and $\mbox{DPX}_{\mbox{PR}}$ complement each other
and are used together in some way.
In particular, $\mbox{DPX}_{\mbox{PR}}$ can be used to find improvements when $\mbox{DEC}_{\mbox{PR}}$ fails.

We note that the points of the Pareto set and its approximations obtained by MOGAs
are appeared to be ``almost uniformly'' distributed along principal diagonal of a rectangle in all test problems
(see e.~g. one result of our MOGA for series S12[1,10][1,20] on Figure).

\begin{figure}[!h] \centering
\begin{center}
\includegraphics[height=4.2cm]{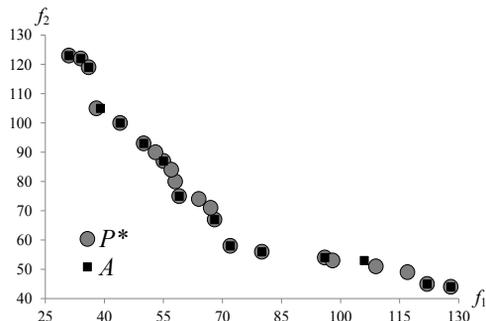}
\caption{Approximation of the Pareto set on one instance of series S12[1,10][1,20] \\ ($\mbox{GD}=0.249$, $\mbox{IGD}=0.566$)} \label{fig:ApproxPS}
\end{center}
\end{figure}

\textbf{\textit{5.2. Pareto set reduction.}}
The reduction of the Pareto set approximation was tested on
the following series with {$n=50$}: S50[1,10][1,10], S50[1,20][1,20], S50[1,10][1,20], S50contr[1,2][1,2].
The series are constructed randomly in the same way as series with {$n=12$} from subsection~5.1.
Each series consists of five instances.
We also took seven ATSP instances of series ftv from TSPLIB library~[17]:
ftv33, ftv35, ftv38, ftv44, ftv47, ftv55, ftv64.
The ftv collection includes instances from vehicle routing applications.
These instances compose series denoted by SftvRand, and their arc weights are used for the first criterion.
The arc weights for the second criterion are generated randomly from interval $[1,d_1^{\max}]$,
where $d_1^{\max}$ is the maximum arc weight on the first criterion.
The population size $N$ was set to $100$. To construct an approximation of the Pareto set $A$ for each instance we run
NSGA-II-DEC once and the run continued for $5000$ iterations.
The only exception is that for series S50contr[1,2][1,2] the run continued for $500$ iterations,
and it was sufficient to find all $51$ points of the Pareto set.

We compare the following cases:
1) when the 1st criterion is more important than the 2nd criterion (1st-2nd case) with $\theta_{12}$;
2) vice versa situation with $\theta_{21}$ (2nd-1st case);
3) 1st-2nd and 2nd-1st cases simultaneously (1st-2nd \& 2nd-1st case).
The degree of the reduction of the Pareto set approximation was investigated
with respect to coefficient of relative importance varying from $0.1$ to $0.9$ by step $0.1$.
On all instances and in all cases for each value of $\theta_{12}$ and $\theta_{21}$  we re-evaluate the obtained approximation
in terms of ``new'' vector criterion $\hat{D}$ upon the formulae from Corollary~\ref{cl_red} and Theorem~\ref{th_red_2}.
Then by the complete enumeration we find the Pareto set approximation
in ``new'' criterion space that gives us the reduction of the Pareto set approximation in the initial criterion space.
\looseness=-1

The number $N^A$ of elements of the Pareto set approximation $A$ and
the percentage of the excluded elements from set $A$ with various values of $\theta_{12}$ and $\theta_{21}$
are presented on average over series in Tables~\ref{tab:S50_1-10_1-10_S50_1-20_1-20} and~\ref{tab:S50_1-10_1-20_SftvRand}.

\begin{table}[!h] \centering
\caption{\textbf{Reduction of the Pareto set approximation for S50[1,10][1,10] and S50[1,20][1,20]}}
\label{tab:S50_1-10_1-10_S50_1-20_1-20}
{\footnotesize
\begin{tabular}{|c|c|c|c|c|c|c|c|c|c|c|c|c|}
\cline{1-11}
\multicolumn{11}{|c|}{{S50[1,10][1,10], \ \ $N^A = 42.4$}}  \\
\cline{1-11}
\multirow{2}{*}{$\theta_{21}$} & \multicolumn{10}{c|}{$\theta_{12}$} & \multicolumn{2}{c}{} \\
\cline{2-11}
& {\scriptsize 2nd-1st} &0.1&0.2&0.3&0.4&0.5&0.6&0.7&0.8&0.9 \\
\cline{1-11}
%\multicolumn{12}{|c|}{\bf The first criterion is more important than the second criterion}\\
%\hline
{\scriptsize 1st-2nd} &  & 7.8 & 22.8 & 42.3 & 51 & 64.3 & 70.6 & 77.1 & 88.1 & 96.7 \\
\cline{1-13}
0.1 & 10.7 & 18 & 33.1 & 51.4 & 59.3 &  72 & 77.1 & 82 & 91 & & 95.8 & 0.9 \\
\cline{1-13}
0.2 & 28.4 & 35.2 & 49.8 &   67.8 & 74.7 & 85.9 & 89.2 & 90.7 & & 95.3 & 92.3 & 0.8 \\
\cline{1-13}
0.3 & 40.3 & 46.6 & 60.3 &  77.5 & 83.5 & 93.9 & 96.6 & & 93.7 & 89.5 & 86 & 0.7 \\
\cline{1-13}
0.4 & 50.4 & 56.8 &   69.6 & 84.7 &  89.4 & 95.3 & & 94.6 & 85.6 & 78.4 & 73.5 & 0.6 \\
\cline{1-13}
0.5 & 63.3 & 69.6 &   81 & 93.5 & 95.9 & & 96.6 & 90.9 & 80.5 & 72.2 & 66.1 & 0.5 \\
\cline{1-13}
0.6 & 69.4 & 75.3 &   85.1 & 95.6 & & 96.5 & 93.6 & 87.1 & 73.3 & 63.1 & 56.2 & 0.4 \\
\cline{1-13}
0.7 & 77.8 & 82.8 & 89.6 & & 95.1 & 92.1 & 87.7 & 80.2 & 65.7 & 52.2 & 45 & 0.3 \\
\cline{1-13}
0.8 & 90.9 & 94.9 &  & 94.9 & 91.1 & 86.7  & 81.1 & 71.4 & 56.2 & 41.7 & 33.9 & 0.2 \\
\cline{1-13}
0.9 & 95.6 &  & 94.6  & 87.1 & 80.9 & 75.3 & 67.8 & 57 & 38.5 & 22.2 & 14.1 & 0.1  \\
\cline{1-13}
\multicolumn{2}{c|}{} & 97 &  91.8 &  81.9 &  72.2 &  65.1 &  55.2 & 43.6 &  24.8 & 8.08 & & {\scriptsize 1st-2nd}  \\
\cline{3-13}
\multicolumn{2}{c|}{} & 0.9 &  0.8 &  0.7 &  0.6 &  0.5 &  0.4 & 0.3 &  0.2 &  0.1 & {\scriptsize 2nd-1st} & \multirow{2}{*}{$\theta_{21}$}  \\
\cline{3-12}
\multicolumn{2}{c|}{} & \multicolumn{10}{c|}{$\theta_{12}$} & \\
\cline{3-13}
\multicolumn{2}{c|}{} & \multicolumn{11}{|c|}{{S50[1,20][1,20], \ \ $N^A = 52.6$}}  \\
\cline{3-13}
\end{tabular}
}
\end{table}

\begin{table}[!h] \centering
\caption{\textbf{Reduction of the Pareto set approximation for S50[1,10][1,20] and SftvRand}}
\label{tab:S50_1-10_1-20_SftvRand}
{\footnotesize
\begin{tabular}{|c|c|c|c|c|c|c|c|c|c|c|c|c|}
\cline{1-11}
\multicolumn{11}{|c|}{{S50[1,10][1,20], \ \ $N^A = 54.4$}}  \\
\cline{1-11}
\multirow{2}{*}{$\theta_{21}$} & \multicolumn{10}{c|}{$\theta_{12}$} & \multicolumn{2}{c}{} \\
\cline{2-11}
& {\scriptsize 2nd-1st} &0.1&0.2&0.3&0.4&0.5&0.6&0.7&0.8&0.9 \\
\cline{1-11}
%\multicolumn{12}{|c|}{\bf The first criterion is more important than the second criterion}\\
%\hline
{\scriptsize 1st-2nd} &  & 10.9 & 24 & 38.2 & 45.5 & 56.4 & 60 & 74.6 & 85.5 & 96.4 \\
\cline{1-13}
0.1 & 21.1 & 26.6 & 36.6 & 50.1 & 60.2 &  70.7 & 75.1 & 83.8 & 90.8 & & 97.8 & 0.9 \\
\cline{1-13}
0.2 & 42.3 & 47.4 & 57.5 & 69.9 & 77.8 & 87.2 & 90.1 & 95.6 & & 97.6 & 97.1 & 0.8 \\
\cline{1-13}
0.3 & 55.1 & 59.9 & 68.8 & 80 & 85.9 & 92.3 & 94.5 & & 94.9 & 92.5 & 91.6 & 0.7 \\
\cline{1-13}
0.4 & 68 & 72 &   78.3 & 86.7 &  91.3 & 96.3 & & 96.4 & 92.3 & 88.3 & 87.2 & 0.6 \\
\cline{1-13}
0.5 & 81.1 & 82.9 &   87.3 & 93.8 & 96.3 & & 95.1 & 92.4 & 87.5 & 83.2 & 81.4 & 0.5 \\
\cline{1-13}
0.6 & 86.9 & 88 & 91.7 & 96.7 & & 96.4 & 91 & 86.9 & 80.6 & 75.4 & 72.2 & 0.4 \\
\cline{1-13}
0.7 & 94.1 & 94.5 & 96.3 & & 94.8 & 89.6 & 82 & 77.3 & 70 & 63.6 & 59.2 & 0.3 \\
\cline{1-13}
0.8 & 97.4 & 97.4 &  & 93.7 & 87.1 & 78.5 & 68 & 61.8 & 52.2 & 44.3 & 39.3 & 0.2 \\
\cline{1-13}
0.9 & 98.2 &  & 92  & 81.7 & 74.2 & 62.3 & 50.4 & 42.2 & 31.8 & 22.7 & 17 & 0.1  \\
\cline{1-13}
\multicolumn{2}{c|}{} & 96 &  89.3 &  75.2 &  65.5 &  52.2 &  38.1 & 28.2 &  16 & 6.1 & & {\scriptsize 1st-2nd}  \\
\cline{3-13}
\multicolumn{2}{c|}{} & 0.9 &  0.8 &  0.7 &  0.6 &  0.5 &  0.4 & 0.3 &  0.2 &  0.1 & {\scriptsize 2nd-1st} & \multirow{2}{*}{$\theta_{21}$}  \\
\cline{3-12}
\multicolumn{2}{c|}{} & \multicolumn{10}{c|}{$\theta_{12}$} & \\
\cline{3-13}
\multicolumn{2}{c|}{} & \multicolumn{11}{|c|}{{SftvRand, \ \ $N^A = 59$}}  \\
\cline{3-13}
\end{tabular}
}
\end{table}

Each table contains data for two series, one is upper secondary diagonal, another is lower secondary diagonal.
The first column stands for the values of $\theta_{21}$, and the second column represents the percentage of the excluded elements
for 2nd-1st case with corresponding values of coefficient.
The third line stands for the values of $\theta_{12}$, and the fourth line represents the percentage of the excluded elements
for 1st-2nd case with corresponding values of coefficient.
Another data of the upper table show the percentage of the excluded elements for 1st-2nd \& 2nd-1st case
(cell at the intersection of corresponding values of $\theta_{12}$ and $\theta_{21}$ varying from $0.1$ to $0.9$).
Interpretation of data of lower table is symmetric to the upper table.

Firstly considered the results of 1st-2nd and 2nd-1st cases.
As seen from Table~\ref{tab:S50_1-10_1-10_S50_1-20_1-20}, for series S50[1,10][1,10] and S50[1,20][1,20]
when $\theta_{12}=0.5$ or $\theta_{21}=0.5$ approximately $65 \%$ of elements of the set $A$ are excluded,
and when $\theta_{12}=0.8$ or $\theta_{21}=0.8$ approximately $10 \%$ of elements are remained.
On these series the results for $\theta_{12}=\theta_{21}$ are similar for both 1st-2nd and 2nd-1st cases.
Note that here the difference between the maximum and minimum
values of set $A$ on the 1st criterion  is almost identical
to the difference on the 2nd criterion for all instances.
\looseness=-1

Series SftvRand and S50[1,10][1,20] show different results: in the 1st-2nd case the reduction occurs ``almost uniformly'', i.e.
the value of $\theta_{12}$ is almost proportional to the degree of the reduction,
in the 2nd-1st case the condition $\theta_{21} = 0.5$ gives approximately $80 \%$ of the excluded elements.
Also, we note that   the percentage of the excluded elements
in the 2nd-1st case for $\theta_{21} = 0.5$ is approximately $1.5$ times
greater than the percentage of the excluded elements in the 1st-2nd case for $\theta_{12}=0.5$.
We suppose that this is due to the difference between the maximum and minimum
values of set $A$ on the 1st criterion  is at least $1.5$ times smaller
than the difference on the 2nd criterion for all instances.
\looseness=-1

Secondly we study the results of 1st-2nd \& 2nd-1st case.
Let us fix some value of $\theta_{12}$, when the 1st criterion is more important than the 2nd one,
and consider the values of percentage of the excluded elements varying $\theta_{21}$ in the feasible interval,
when the 2nd criterion is more important than the 1st one.
For example, if we put $\theta_{12} = 0.4$ for series S50[1,10][1,10], then varying $\theta_{21}$ from $0.1$ till $0.5$
we get the column $59.3$, $74.7$, $83.5$, $89.4$, $95.9$. Also we investigate vice versa situation,
when we fix some value of $\theta_{21}$, change values of $\theta_{12}$, and get the corresponding line.
According Table~\ref{tab:S50_1-10_1-10_S50_1-20_1-20} for both series S50[1,10][1,10] and S50[1,20][1,20] we conclude, that
the ratio between percentage of the excluded elements for some fixed value of $\theta_{12}$ and percentage, %at the line
when we fix $\theta_{21}$ with the same value, is approximately equal to 1.
Thus the reduction occurs ``almost uniformly''.

From Table~\ref{tab:S50_1-10_1-20_SftvRand} we see, that for series S50[1,10][1,20]
the ratio of percentage of the excluded elements at the column with fixed $\theta_{12} = 0.1$ % at the line with
over percentage at the line with fixed $\theta_{21} = 0.1$
changes from $1.29$ till $1.07$. When we fix $\theta_{12} = 0.6$ and $\theta_{21} = 0.6$,
the ratio changes from $0.85$ till $0.98$. In total, the average ratio is greater than 1 for fixed
$0.1 \leqslant \theta_{12} = \theta_{21} \leqslant 0.2$,
and the average ratio is smaller than 1 for fixed $0.3 \leqslant \theta_{12} = \theta_{21} \leqslant 0.8$.

For series SftvRand we conclude that the ratio of percentage of the excluded elements
when we fix $\theta_{12} = 0.1$ to percentage when $\theta_{21} = 0.1$
changes from $1.39$ till $1.06$. When we fix $\theta_{12} = 0.6$ and $\theta_{21} = 0.6$,
the ratio changes from $0.84$ till $0.98$.
The average ratio for various fixed $\theta_{12}=\theta_{21}$ has the same tendency
as for the series S50[1,10][1,20].

Thus, we investigate symmetric situations:
1) when the 1st criterion is more important than the 2nd one with fixed $\theta_{12}$
and the 2nd criterion is more important than the 1st one with varied $\theta_{21}$;
2) when the 2nd criterion is more important than the 1st one with fixed $\theta_{21}$
and the 1st criterion is more important than the 2nd one with varied $\theta_{12}$.
We can state that more elements of the Pareto set are excluded
with a small fixed coefficient of relative importance (no more than $0.2$),
when we analyze situation 1) in comparison with situation 2) for both series S50[1,10][1,20] and SftvRand.
Otherwise, more elements of the Pareto set are eliminated,
when we fix coefficient of relative importance at a medium or a high value
(at least $0.3$) in situation 2) compared to situation 1).

As a result we conclude that in the 1st-2nd \& 2nd-1st case the difference between the maximum and minimum
values of set $A$ on criteria also influences on the degree of reduction.

The results of the experiment on series S50contr[1,2][1,2] confirm
the theoretical results of subsection~3.2 (Theorems~\ref{prop1_crit_theta}--\ref{thm3_crit_theta}).
In 1st-2nd and 2nd-1st cases we do not have a reduction, when coefficient of relative importance is lower than $0.5$,
and the reduction up to one element takes place if the otherwise inequality occurs.
In 1st-2nd \& 2nd-1st case, if $\theta_{12} \geqslant 0.5$ or $\theta_{21} \geqslant 0.5$, is valid, then
the reduction of the Pareto set approximation consists of only one element.
If both inequalities $\theta_{12} < 0.5$ and $\theta_{21} < 0.5$ hold, then the reduction does not occur.

The results for series with $n=12$ from the previous subsection are analogous.
Based on the results of the experiment we suppose that
the degree of the reduction of the Pareto set approximation will be similar
for the large-size problems with the same structure as the considered instances.
\looseness=-1

\textbf{6. Conclusion.}
We applied to the bi-ATSP  the axiomatic approach of the Pareto set reduction proposed by V.~Noghin.
For particular cases the series of ``quanta of information''
that guarantee the reduction of the Pareto set were identified.
An approximation of the Pareto set to the bi-ATSP was found by
a generational multi-objective genetic algorithm with new crossovers involving the Pareto relation.
The experimental evaluation indicated the degree of reduction of the Pareto set approximation
for various problem structures in the case of one and two ``quanta of information''.

Further research may include construction and analysis of new classes of multicriteria ATSP instances
with complex structures of the Pareto set. In particular, bi-ATSP with  objectives of different type
(for example, the first criterion is the sum of arc weights,
and the second one is the maximum arc weight).
It is also important to consider real-life ATSP instances with real-life decision maker
 and investigate effectiveness of the axiomatic approach for them.
Moreover, developing a faster implementation of the multi-objective genetic algorithm %with steady-state replacement
using more effective non-domination sorting, combination of exploitive and explorative crossovers,
 and local search procedures has great interest.

\

{\small \noindent \textbf{References}

\

1. Ausiello G., Crescenzi P., Gambosi G., Kann V., Marchetti-Spaccamela A.,
  Protasi M. \textit{Complexity and Approximation}. Berlin, Heidelberg, Springer-Verlag Publ., 1999, 524~p.

2. Ehrgott M. \textit{Multicriteria optimization}. Berlin, Heidelberg, Springer-Verlag Publ., 2005, 323~p.
\looseness=-1

3. Podinovskiy V.~V., Noghin V.~D. \textit{Pareto-optimal'nye resheniya
  mnogokriterial'nyh zadach} [\textit{Pareto-optimal solutions of multicriteria
  problems}]. Moscow, Fizmatlit Publ., 2007, 256~p. (In Russian)

4. Figueira J.~L., Greco S., Ehrgott M. \textit{Multiple criteria decision analysis:
  state of the art surveys}. New York, Springer-Verlag Publ., 2005, 1048~p.

5. Noghin V.~D. \textit{Reduction of the Pareto Set: An Axiomatic Approach}.
  Cham, Springer Intern. Publ., 2018, 232~p.

6. Klimova O.~N. The problem of the choice of optimal chemical composition of
  shipbuilding steel. \textit{Journal of Computer and Systems Sciences International},
  2007, vol.~46(6), pp. 903--907.

7. Noghin V.~D., Prasolov A.~V. The quantitative analysis of trade policy: a
  strategy in global competitive conflict.
  \textit{Intern. Journal of Business Continuity and Risk
  Management}, 2011, vol.~2(2), pp.~167--182.
\looseness=-1

8. Angel E., Bampis E., Gourv\'es L., Monnot J. (Non)-approximability for
  the multicriteria TSP(1,2). \textit{Fundamentals of Computation Theory 2005: 15th Intern. Symposium.
  Lecture Notes in Computer Science.} Vol.~3623. Lubeck, Germany, 2005. pp.~329--340.

9. Buzdalov M., Yakupov I., Stankevich A. Fast implementation of the
  steady-state NSGA-II algorithm for two dimensions based on incremental
  non-dominated sorting. \textit{Proceedings of the 2015 Annual conference on
  Genetic and Evolutionary Computation (GECCO-15)}. Madrid, Spain, 2015, pp.~647--654.
\looseness=-1

10. Deb K., Pratap A., Agarwal S., Meyarivan T. A fast and elitist
  multi-objective genetic algorithm: NSGA-II. \textit{IEEE Transactions on
  Evolutionary Computation}, 2002, vol.~6(2), pp.~182--197.
\looseness=-1

11. Li H., Zhang Q. Multiobjective optimization problems with complicated Pareto
  sets, MOEA/D and NSGA-II. \textit{IEEE Transactions on Evolutionary Computation},
  2009, vol.~13(2), pp.~284--302.
\looseness=-1

12. Yuan Y., Xu H., Wang B. An improved NSGA-III procedure for evolutionary
  many-objective optimization. \textit{Proceedings of the 2014 Annual conference on
  Genetic and Evolutionary Computation (GECCO-14)}. Vancouver, BC, Canada, 2014, pp.~661--668.

13. Zitzler E., Brockhoff D., Thiele L. The hypervolume indicator revisited: On
  the design of Pareto-compliant indicators via weighted integration.
  \textit{Proceedings of conference on Evolutionary Multi-Criterion Optimization,
  Lecture Notes in Computer Science}. Vol.~4403.
  Berlin, Springer Publ., 2007, pp.~862--876.

14. Zitzler E., Laumanns M., Thiele L. SPEA2: Improving the strength
  Pareto evolutionary algorithm. \textit{Evolutionary Methods
  for Design, Optimization and Control with Application to Industrial Problems:
  Proceedings of EUROGEN 2001 conference}.
  Athens, Greece, 2001, pp.~95--100.

15. Garcia-Martinez C., Cordon O., Herrera F. A taxonomy and an empirical
  analysis of multiple objective ant colony optimization algorithms for the
  bi-criteria TSP. \textit{European Journal of Operational Research}, 2007, vol.~180,
  pp.~116--148.

16. Psychas I.~D., Delimpasi E., Marinakis Y. Hybrid evolutionary algorithms for
  the multiobjective traveling salesman problem. \textit{Expert Systems with
  Applications}, 2015, vol.~42(22), pp.~8956--8970.

17. Reinelt G. TSPLIB -- a traveling salesman problem library. \textit{ORSA Journal on
  Computing}, 1991, vol.~3(4), pp.~376--384.

18. Zakharov A.~O., Kovalenko Yu.~V. Reduction of the Pareto set in bicriteria
  assymmetric traveling salesman problem.
  \textit{OPTA-2018. Communications in Computer and Information Science}. Vol.~871.
  Eds by A.~Eremeev, M.~Khachay, Y.~Kochetov, P.~Pardalos.
  Cham, Springer Intern. Publ., 2018, pp.~93-105.

19. Noghin V.~D. Reducing the Pareto set algorithm based on an arbitrary finite
  set of information ``quanta''. \textit{Scientific and Technical Information Processing},
  2014, vol.~41(5), pp.~309--313.

20. Klimova O.~N., Noghin V.~D. Using interdependent information on the relative
  importance of criteria in decision making. \textit{Computational Mathematics and Mathematical Physics},
  2006, vol.~46(12), pp.~2080--2091.

21. Noghin V.~D. Reducing the Pareto set based on set-point information.
  \textit{Scientific and Technical Information Processing}, 2011, vol.~38(6), pp.~435--439.

22. Zakharov A.~O. Pareto-set reduction using compound information of a closed
  type. \textit{Scientific and Technical Information Processing}, 2012, vol.~39(5), pp.~293--302.

23. Emelichev V.~A., Perepeliza V.~A. Complexity of vector optimization problems
  on graphs. \textit{Optimization: A Journal of Mathematical Programming and Operations
  Research}, 1991, vol.~22(6), pp.~906--918.
\looseness=-1

24. Vinogradskaya T.~M., Gaft M.~G. Tochnaya verhn'ya otzenka chisla
  nepodchinennyh reshenii v mnogokriterial'nyh zadachah [The least upper
  estimate for the number of nondominated solutions in multicriteria
  problems]. \textit{Avtomatica i Telemekhanica} [\textit{Automation and Remote Control}],
  1974, vol.~9, pp. 111--118. (In Russian)

25. Reeves C.~R. Genetic algorithms for the operations researcher. \textit{INFORMS Journal
  on Computing}, 1997, vol.~9(3), pp.~231--250.

26. Eremeev A.V., Kovalenko Y.V. Genetic algorithm with optimal recombination
  for the asymmetric travelling salesman problem. \textit{Large-Scale Scientific Computing 2017.
  Lecture Notes of Computer Science}, 2018, vol.~10665, pp.~341--349.

27. Whitley D., Starkweather T., McDaniel S., Mathias K. A comparison of
  genetic sequencing operators. \textit{Proceedings of the Fourth Intern.
  conference on Genetic Algorithms}. San Diego, California, USA, 1991, pp.~69--76.

28. Radcliffe N.~J. The algebra of genetic algorithms. \textit{Annals of Mathematics and
  Artificial Intelligence}, 1994, vol.~10(4), pp.~339--384.

29. Jaszkiewicz A., Zielniewicz P. Pareto memetic algorithm with path relinking
  for bi-objective traveling salesperson problem. \textit{European Journal of
  Operational Research}, 2009, vol.~193, pp.~885--890.
\looseness=-1

30. Kumar R., Singh P.~K. Pareto evolutionary algorithm hybridized with local search
  for bi-objective TSP.
  \textit{Hybrid Evolutionary Algorithms. Studies in Computational Intelligence.} Vol.~75.
  Eds by A.~Abraham, C.~Grosan, H.~Ishibuchi.
   Berlin, Heidelberg, Springer Publ., 2007, pp.~361--398.

31. Lust T., Teghem J. The Multiobjective traveling salesman problem: A survey
  and a new approach.
  \textit{Advances in Multi-objective Nature Inspired
  Computing. Studies in Computational Intelligence.} Vol.~272.
  Eds by C.~A.~Coello Coello, C.~Dhaenens, L.~Jourdan.
  Berlin, Heidelberg, Springer Publ., 2010, pp.~119--141.

32. {\it Multiobjective optimization library.} URL: http://home.ku.edu.tr/$\sim$moolibrary/ (accessed: 09.02.2018).
\looseness=-1

\end{document}